\documentclass{llncs}

\usepackage{epsfig}
\usepackage{myalgo,amssymb,amsmath}
\usepackage{subfigure}
\usepackage{color}

\usepackage{makeidx}  
\usepackage{graphics}  
\input{psfig.sty}

\begin{document}
\frontmatter          
\pagestyle{headings}  
\addtocmark{Data stream mining} 
\vspace{1cm}

\mainmatter

\title{Discriminative Subnetworks with Regularized Spectral Learning for Global-state Network Data}

\author{Xuan Hong Dang, Ambuj K. Singh,  Petko Bogdanov, Hongyuan You and Bayyuan Hsu}
\institute{ Department of Computer Science, University of
California Santa Barbara, USA
\email{\{xdang,ambuj,petko,hyou,soulhsu\}@cs.ucsb.edu} }

\maketitle \vspace{-5ex}
\begin{abstract}
Data mining practitioners are facing challenges from data with
network structure. In this paper, we address a specific class of
\textit{global-state} networks which comprises of a set of network
instances sharing a similar structure yet having different values
at local nodes. Each instance is associated with a global state 
which indicates the occurrence of an event. The objective is
to uncover a small set of discriminative subnetworks that can
optimally classify global network values. Unlike most existing
studies which explore an exponential subnetwork space, we address
this difficult problem by adopting a space transformation
approach. Specifically, we present an algorithm that optimizes a
constrained dual-objective function to learn a low-dimensional
subspace that is capable of discriminating networks labelled by
different global states, while reconciling with common network
topology sharing across instances. Our algorithm takes an
appealing approach from spectral graph learning and
we show that the globally optimum solution can be achieved via
matrix eigen-decomposition. \vspace{-2ex}
\end{abstract}


\section{Introduction}                                    \label{sec:Intro}
\vspace{-1ex}

With the increasing advances in hardware and software technologies
for data collection and management, practitioners in data mining
are now confronted with more challenges from the collected
datasets: the data are no longer as simple as objects with
flattened representation but now embedded with relationships among
variables describing the objects. This sort of data is often
referred to as \textit{network} or \textit{graph} data. In the
literature, there are a large number of techniques developed to
mine useful patterns from network databases, ranging from frequent
(sub)networks mining~\cite{JiangCZ13}, network
classification/clustering~\cite{AggarwalW10b,KetkarHC09} to
anomaly detection~\cite{AkogluMF10}. Often, even for the same data
mining task, we may need different algorithms to be developed
depending on whether the networks are \textit{directed} or
\textit{indirected}, or whether the data resides at nodes, edges
or both of them~\cite{JiangCZ13}.

In this work, the focus is on a specific class of interesting
networks in which we have a series of network instances that share
a common structure but may have different dynamic values at local
nodes and/or edges. In addition, each network instance is
associated with a global state indicating the occurrence of an
event. Such a class of \emph{global-state} network data can be
used to model a number of real-world applications ranging from
opinion evolution in social networks~\cite{Lee08}, regulatory
networks in biology~\cite{LiL08a} to brain networks in
neuroscience~\cite{DavidsonGCW13}. For example, we possess the
same set of genes (nodes) embedded in regulatory networks. Yet,
research in systems biology shows that the gene expression levels
(node values) may vary across individuals and for some specific
genes, their over-expressions may impact those in the neighbors
through the regulatory network. These local effects may jointly
encode a logical function that determines the occurrence of a
disease~\cite{LiL08a,RanuHS13}. In analyzing these types of
network data, a natural question to be asked is how one can learn
a function that can determine the global-state values of the
networks based on the dynamic values captured at their local nodes
along with the network topology? More specifically, is it possible
to identify a small succinct set of influential discriminative
subnetworks whose local-node values have the maximum impact on the
global states and thus uncover the complex relationships between
local entities and the global-state network properties? In
searching for an answer, obviously, a naive approach would be to
enumerate all possible subnetworks and seek those who have the
most discriminative potential. Nonetheless, as the number of
subnetworks is \textit{exponentially} proportional to the numbers
of nodes and edges, this approach generally is analytically
intractable and might not be feasible for large scale networks. A
more practical approach is to perform heuristic sampling from the
space of subnetworks. Though greatly reducing the number of
subnetworks to be visited, the sampling approaches might still
suffer from suboptimal solutions and might further lose
explanation capability due to the large number of generating
subnetworks.

\vspace{-1ex}

In this paper, we propose a novel algorithm for mining a set of
concise subnetworks whose local-state node values discriminate
networks of different global-state values. Unlike the existing
techniques that directly search through the exponential space of
subnetworks, our proposed method is fundamentally different by
investigating the discriminative subnetworks in a low dimensional
transformed subspace. Toward this goal, we construct on top of the
network database three meta-graphs to learn the network
neighboring relationships. The first meta-graph is built to
capture the network topology sharing across network instances
which serves as the network constraint in our subspace learning
function, whereas the two subsequent ones are build to essentially capture the
relationships between neighboring networks, especially those
located close to the potential discriminative boundary. By this
setting, our algorithm aims to discover a unique low dimensional
subspace to which: i) networks sharing similar global state values
are mapped close to each other while those having different global
values are mapped far apart; ii) the common network topology is
smoothly preserved through constraints on the learning process. In
this way, our algorithm helps to attack two challenging issues at
the same time. It first avoids searching through the original
space of exponential number of subnetworks by learning a single
subspace via the optimization of a single dual-objective function.
Second, our network topology constraint not only matches properly
with our subspace learning function, its quadratic form naturally
imposes the $L_2$-norm shrinkage over the connecting nodes,
resulting in an effective selection of relevant and dominated
nodes for the subnetworks embedded in the induced subspace.
Additionally, the principal technical contributions of our work is
the formulation of our learning objective function that is
mathematically founded on spectral learning and its advantages
therefore not only ensure the stability but also the global
optimum of the uncovered solutions.

In summary, we claim the following contributions: (i)
\textit{Novelty:} We formulate the problem of mining
discriminative subnetworks by transformed subspace learning---an
approach that is fundamentally different from most existing
techniques that address the problem in the original
high-dimensional network space. (ii) \textit{Flexibility:} We
propose a novel dual-objective function along with constraints to
ensure learning of a single subspace in which different global
state networks are well discriminated while smoothly retaining
their common topology. (iii) \textit{Optimality:} We develop a
mathematically sound solution to solve the constrained
optimization problem and show that the optimal solution can be
achieved via matrix eigen-decomposition. (iv) \textit{Practical
relevance:} We evaluate the performance of the proposed technique
on both synthetic and real world datasets and demonstrate its
appealing performance against related techniques in the
literature.

\vspace{-2ex}
\section{Preliminaries and Problem Setting}                               \label{sec:Problem}
\vspace{-2ex}

In this section, we first introduce some preliminaries related to
network data with global state values and then give the definition
of our problem on mining discriminative subgraphs to distinguish
global state networks.

\vspace{-1ex}
\begin{definition} \emph{(Network data instance)} \label{Def1}
\noindent Given $V_i=\{v_1,v_2,\ldots,v_{n_i}\}$ as a set of nodes
and $E_i\subseteq V_i\times V_i$ as a set of edges, each
connecting two nodes $(v_p,v_q)$ if they are known to relate or
influence each other, we define a network instance (or snapshot)
$N_i$ as a quadruple $N_i=(V_i,E_i,L_i,S_i)$ in which $L_i$ is a
function operating on the local states of nodes $L_i: V_i
\rightarrow \mathbb{R}$ and $S_i$ encodes the global network state
of $N_i$.
\end{definition}
\vspace{-2ex}

We consider $N_i$ as an \textit{indirected} network and
values at its local nodes are numerical (both continuous and
binary) while its global state is a discrete value. Since each
$N_i$ is associated with $S_i$ as its state property, $N_i$ is
often referred to as a \textit{global-state} network. For example,
in the gene expression data, each $N_i$ corresponds to a subject
and a local state indicates the gene expression level at node $v_p
\in V_i$ whereas the global state encodes the presence or absence
of the disease, i.e., $S_i\in \{presence, absence\}$. Likewise in
a dynamic social network, a value at each node $v_p$ may encode
the political standpoint of an individual whereas the global state
indicates the overall political viewpoint of the entire community
at some specific time (snapshot). Both local and global states may
change across different network snapshots. Note that, for network instances/snapshots with different structures, we may
use the null value to denote the state of a missing node and
consequently, an edge in a network instance is valid only if it
connects two non-null nodes. 

Now, let us consider a database
consisting $m$ network instances $\mathbb{N}=\{N_1,N_2,\ldots
,N_m\}$, we further define the following network over these
network instances:

\vspace{-1ex}
\begin{definition} \emph{(Generalized network - first meta-graph)} \label{Def2}

\noindent We define the generalized network $N$ as a triple
$N=(V,E,K)$ where $V=V_1\cup V_2 \ldots \cup V_m$ and if $\exists
(v_p,v_q) \in E_i$, such an edge also exists in $E$. For a valid
edge $E(p,q) \in E$, we associate a weight $K(p,q)$ as the
fraction of network instances having edge $E(p,q)$ in their
topology structure,i.e., $K(p,q)=m^{-1}\times \sum_i E_i(p,q)$
with $E_i(p,q)=1$ if there exists an edge between $v_p$, $v_q$ in
network $N_i$. As such, $K(p,q)$ is naturally normalized between
$(0,1]$. The value of 1 means the corresponding edge exists in all
$N_i$'s while a value close to 0 shows that the edge only exists
in a small fraction of network data.
\end{definition}
\vspace{-2ex}

It should be noted here that while we have no edge values at \\
individual networks $N_i$'s, we have non-zero value associated with
each existing edge $E(p,q)$ in the generalized network $N$.
Indeed, the corresponding $K(p,q)$ reflects how frequently there is an edge between
$v_p$ and $v_q$ or equivalently, how strongly is the mutual
influence between two entities $v_p$ and $v_q$ across all
networks. As $N$ is defined based on all network instances, we
also view $N$ as our first meta-graph with $V$ being its vertices
and $K$ capturing its graph topology generalized from the network
topology of all network instances. We are now ready to define our
problem as follows.

\vspace{-1ex}
\begin{definition} \emph{(Mining Discriminative Subnetworks Problem)} \label{Def3}

\noindent Given a database of network data instances/snapshots
$\mathbb{N}=\{N_1,N_2,\ldots,N_m\}$, we aim to learn an optimal
and succinct set of subnetworks with respect to the topology
structure generalized in the first meta-graph that well
discriminate network instances with different global state values.
\end{definition}
\vspace{-2ex}

\vspace{-2ex}
\section{Our approach}                                    \label{sec:Approach}
\vspace{-2ex}
\subsection{Meta-Graphs over Network Instances}           \label{sec:MetaGraph}
\vspace{-1ex} As mentioned in the above sections, searching for
optimal subnetworks in the fully high dimensional original network
space is always challenging and potentially intractable. We adopt
an indirect yet more viable approach by transforming the original
space into a low dimensional space of which networks with
different global-states are well distinguished while concurrently
retaining the generalized network topology captured by our first
meta-graph. Toward this goal, we develop two neighboring
\textit{meta-graphs} based on both the local state values and
global state values.

We denote these two meta-graphs respectively by $G^{+}$ and
$G^{-}$. Their vertices correspond to the network instances while
a link connecting two vertices represents the neighboring
relationship between two corresponding network instances. For the
meta-graph $G^{+}$, we denote $\mathbf{A}^{+}$ as its affinity
matrix that captures the similarity of neighboring networks having
the same global state values. Likewise, we denote $\mathbf{A}^{-}$
as the affinity matrix for meta-graph $G^{-}$ that captures the
similarity of neighboring networks yet having different global
network states. As such, $\mathbf{A}^{+}$ and $\mathbf{A}^{-}$
respectively encode the weights on the vertex-links of two
corresponding graphs $G^{+}$ and $G^{-}$. In computing values for
these affinity matrices, with each given network instance $N_i$,
we find its $k$ nearest neighboring networks based on the local
state values and divide them into two sets, those sharing similar
global state values and those having different global states. More
specifically, let $k$NN($N_i$) be the neighboring set of $N_i$,
then elements of $\mathbf{A}^{+}$ and $\mathbf{A}^{-}$ are
computed as:
$\mathbf{A}^{+}_{ij}= \frac{\mathbf{v}_i \cdot
\mathbf{v}_j}{\|\mathbf{v}_i\|\|\mathbf{v}_j\|}$ if $S_i=S_j$ and
$N_j\in k$NN($N_i$) or $N_i\in k$NN($N_j$), otherwise we set
$\mathbf{A}^{+}_{ij}=0$.
And $\mathbf{A}^{-}_{ij}= \frac{\mathbf{v}_i \cdot
\mathbf{v}_j}{\|\mathbf{v}_i\|\|\mathbf{v}_j\|}$ if $S_i\neq S_j$
and $N_j\in k$NN($N_i$) or $N_i\in k$NN($N_j$), otherwise
$\mathbf{A}^{-}_{ij}=0$.
In these equations, we have denoted the boldface letters
$\mathbf{v}_i$ and $\mathbf{v}_j$ as the vectors encoding the
dynamic local states of $N_i$'s and $N_j$'s nodes, and have used
the cosine distance to define the similarity between two network
instances. It is worth mentioning that, though existing other
measures for network data~\cite{SS2013}, our using of cosine
distance is motivated by the observation that we can view each
node as a single feature and thus the network data can be
essentially considered as a special case of very high dimensional
data. As such, the symmetric cosine measure can be
effectively used though obviously the other ones~\cite{SS2013} can
also be directly applied here.

It is also important to give the intuition behind our above
computation. First, notice that both $\mathbf{A}^{+}$ and
$\mathbf{A}^{-}$ are the affinity matrices having the same size of
$m\times m$ since we calculate for every network instance. Second,
while $\mathbf{A}^{+}$ captures the similarity of network
instances sharing the same global states and neighboring to each
other, $\mathbf{A}^{-}$ encodes the similarity of different global
state networks yet also neighboring to each other. Such networks
are likely to locate close to the discriminative boundary function
and thus they play essential roles in our subsequent learning
function. Third, both $\mathbf{A}^{+}$ and $\mathbf{A}^{-}$ are
sparse and symmetric matrices since only $k$ neighbors are
involved in computing for each network and if $N_j$ is neighboring
to $N_i$, we also consider the inverse relation, i.e., $N_i$ is
neighboring to $N_j$. Moreover, $\mathbf{A}^{-}$ is generally
sparser compared to $\mathbf{A}^{+}$ as the immediate observation
from the second remark.

\vspace{-2ex}%
\subsection{Constrained Dual-Objective Function}           \label{sec:ObjFunc}
\vspace{-1ex}%

Let us recall that $\mathbf{v}_i$ is the vector encoding the node
states of the corresponding network $N_i$ and let us denote the
transformation function that maps $\mathbf{v}_i$ into our novel
target subspace by $f(\textbf{v}_i)$. We first formulate the two
objective functions as follows: \vspace{-3ex}
\begin{align} \label{F-Obj1}
 \operatorname*{arg\,min}_{f} \sum_{i=1}^m\sum_{j=1}^m
(f(\textbf{v}_i)- f(\textbf{v}_j))^2 \mathbf{A}^{+}_{ij}
\end{align}
\vspace{-4ex}
\begin{align} \label{F-Obj2}
 \operatorname*{arg\,max}_{f} \sum_{i=1}^m\sum_{j=1}^m
(f(\textbf{v}_i)- f(\textbf{v}_j))^2 \mathbf{A}^{-}_{ij}
\end{align}
\vspace{-3ex}

To gain more insights into these setting objectives, let us take a closer
look at the first Eq.\eqref{F-Obj1}. If two network instances
$N_i$ and $N_j$ have similar local states in the original space
(i.e., $\mathbf{A}^{+}_{ij}$ is large), this first objective
function will be penalized if the respective points
$f(\textbf{v}_i)$ and $f(\textbf{v}_j)$ are mapped far part in the
transformed space. As such, minimizing this cost function is
equivalent to maximizing the similarity amongst instances having
the same global network states in the reduced dimensional
subspace. On the other hand, looking at Eq.\eqref{F-Obj2} can tell
us that the function will incur a high penalty (proportional to
$\mathbf{A}^{-}_{ij}$) if two networks having different global
states are mapped close in the induced subspace. Thus, maximizing
this function is equivalent to minimizing the similarity among
neighboring networks having different global states in the novel
reduced subspace. As mentioned earlier, such networks tend to
locate close to the discriminative boundary function and hence,
maximizing the second objective function leads to the maximal
margin among clusters of different global-state networks.

Having the mapping function $f(.)$ to be optimized above, it is
crucial to ask which is an appropriate form for it. Either a
linear or non-linear function can be selected as long as it
effectively optimizes two objectives concurrently. Nonetheless,
keeping in mind that our ultimate goal is to derive a set of
succinct discriminative subnetworks along with their
\textit{explicit} nodes. Optimizing a non-linear function is
generally not only more complex but importantly may lose the
capability in explaining how the new features have been derived
(since they will be the
 \textit{non-linear} combinations of the original nodes).
We therefore would prefer $f(.)$ as in the form of a linear
combination function and following this, $f(.)$ can be represented
explicitly as a transformation matrix $U_{n\times d}$ that
linearly combines $n$ nodes into $d$ novel features ($d\ll n$) of
the induced subspace. For the sake of discussion, we elaborate
here for the projection onto 1-dimensional subspace (i.e., $d=1$).
The solution for the general case $d>1$ will be straightforward
once we obtain the solution for this base case. Given this
simplification and with little algebra, we recast our first
objective function as follows:

\vspace{-4ex}
\begin{align} \label{F-Obj11}
 \operatorname*{arg\,min}_{\mathbf{u}} & \sum_{i=1}^m\sum_{j=1}^m
\|\mathbf{u}^T\mathbf{v}_i- \mathbf{u}^T\mathbf{v}_j\|^2
\mathbf{A}^{+}_{ij} =\sum_{i=1}^m\sum_{j=1}^m
tr\left(\mathbf{u}^T(\mathbf{v}_i-\mathbf{v}_j)(\mathbf{v}_i-\mathbf{v}_j)^T
\mathbf{u} \right) \mathbf{A}^{+}_{ij} \nonumber \\
&= tr\left(\sum_{i=1}^m\sum_{j=1}^m
\left(\mathbf{u}^T(\mathbf{v}_i-\mathbf{v}_j)\mathbf{A}^{+}_{ij}(\mathbf{v}_i-\mathbf{v}_j)^T\right)
\mathbf{u} \right) \nonumber \\
&= 2 tr\left(\mathbf{u}^T \mathbf{V} \mathbf{D}^{+} \mathbf{V}^T
\mathbf{u} \right) - 2tr\left(\mathbf{u}^T \mathbf{V}
\mathbf{A}^{+} \mathbf{V}^T\mathbf{u}\right)= 2
tr\left(\mathbf{u}^T \mathbf{V} \mathbf{L}^{+} \mathbf{V}^T
\mathbf{u} \right)
\end{align}
\vspace{-4ex}

\noindent in which we have used $tr(.)$ to denote the trace of a
matrix and $\mathbf{V}$ as the matrix whose column $i$th
accommodates the dynamic local states of network instance $N_i$
(i.e., $\mathbf{v}_i$), forming its size of $n\times m$. Also,
$\mathbf{D}$ is the diagonal matrix whose
$\mathbf{D}^{+}_{ii}=\sum_j \mathbf{A}^{+}_{ij}$ and we have
defined $\mathbf{L}^{+}=\mathbf{D}^{+}-\mathbf{A}^{+}$, which can
be shown to be the Laplacian matrix~\cite{Golub1996}. For the
second objective function in Eq.\eqref{F-Obj2}, we can repeat the
same computation which yields to the following form: \vspace{-2ex}
\begin{align} \label{F-Obj21}
 \operatorname*{arg\,max}_{\textbf{u}} & \sum_{i=1}^m\sum_{j=1}^m
\|\mathbf{u}^T\mathbf{v}_i- \mathbf{u}^T\mathbf{v}_j\|^2 \mathbf{A}^{-}_{ij} \nonumber \\
&= 2 tr\left(\mathbf{u}^T \mathbf{V} \mathbf{D}^{-} \mathbf{V}^T
\mathbf{u} \right) - 2tr\left(\mathbf{u}^T
\mathbf{V} \mathbf{A}^{-} \mathbf{V}^T\mathbf{u}\right) \nonumber \\
&= 2 tr\left(\mathbf{u}^T \mathbf{V} \mathbf{L}^{-} \mathbf{V}^T
\mathbf{u} \right)
\end{align}
\vspace{-3ex}

\noindent where again $\mathbf{D}^{-}$ is the diagonal matrix with
$\mathbf{D}^{-}_{ii}=\sum_j \mathbf{A}^{-}_{ij}$ and we have
defined $\mathbf{L}^{-}=\mathbf{D}^{-}-\mathbf{A}^{-}$.

Notice that while the above formulations aim at discriminating
different global state networks in the low dimensional subspace,
it has not yet taken into consideration the generalized network
structure captured by our first meta-graph. As described
previously, the mutual interactions among nodes are also important
in determining the global network states. Also according to
Definition~\ref{Def2}, the larger the value placing on the link
between nodes $v_p$ and $v_q$, the more likely they are being
involved in the same process. Therefore, we would expect our
mapping vector $\textbf{u}$ not only separating well different
global state networks but also ensuring its smoothness property
w.r.t. the generalized network topology characterized by the first
meta-graph $N$.

Toward the above objective, we formulate the network topology as a
constraint in our learning objective function, and in order to be
consistent with the approach based on spectral graph analysis, we
encode the topology captured in $N$ by an $n\times n$ constraint
matrix $\mathbf{C}$ whose elements are defined by:

\vspace{-3ex}
\begin{align}
 \mathbf{C}_{pq}=\mathbf{C}_{qp} =
  \begin{cases}
   \sum_q K{(p,q)} & \text{if  $v_p \equiv v_q$ }  \\
   - K(p,q) & \text{if $v_p$ and $v_q$ are connected} \\
   0  & \text{otherwise }
  \end{cases}
\end{align}
\vspace{-3ex}

It is easy to show that, by this definition, $\mathbf{C}$ is also
the Laplacian matrix and its quadratic form, taking $\mathbf{u}$
as the vector, is always non-negative: \vspace{-2ex}
\begin{align} \label{F-Obj211}
\mathbf{u}^T\mathbf{C}\mathbf{u} &= \sum_{p=1}^n u_p^2
\sum_{q=1}^n K(p,q)
 - \sum_{p=1}^n \sum_{q=1}^n u_p u_q K(p,q) \nonumber \\
&= \frac{1}{2} \sum_{p=1}^n \sum_{q=1}^n K(p,q)(u_p - u_q)^2 \geq
0
\end{align}
\vspace{-3ex}

\noindent in which $u_p,~u_q$ are components of vector
$\mathbf{u}$. It is possible to observe that if $K(p,q)$ is large,
indicating nodes $v_p$ and $v_q$ are strongly interacted in large
portion of the network instances, the coefficients of $u_p$ and
$u_q$ should be similar (i.e., smooth) in order to minimize this
equation. From the network-structure perspective, we would say
that if $v_p$ is known as a node affecting the global network
state, its selection in the transformed space will increase the
possibility of being selected of its nearby connected node $v_q$
if $K(p,q)$ is large, leading to the formation of discriminative
subnetworks in the induced subspace. Therefore, in combination
with the dual-objective function formulated above, we finally
claim our constrained optimization problem as follows (the
constants can be omitted due to optimization):
\vspace{-3ex}
\begin{align} \label{obj-Func4}
\mathbf{u}^*=& \underset{\mathbf{u}}{\operatorname{arg\,max~}}
\left\{tr\left(\mathbf{u^T V (L^{-}-L^{+})V}^T \mathbf{u}\right)\right\}\nonumber \\
& \text{~~subject to~~} \mathbf{u}^T \mathbf{C u} \leq t \nonumber \\
& \text{~~~~~~~~and ~~}\mathbf{u}^T
\mathbf{V}\mathbf{D}^{+}\mathbf{V}^T \mathbf{u} = 1
\end{align}
\vspace{-4ex}

The first network topology constraint aims to retain the
smoothness property of $\mathbf{u}$ whereas the second constraint
aims to remove its freedom, meaning that we need $\mathbf{u}$'s
direction rather than its magnitude. The network topology
constraint is beneficial in two ways. First as presented above, it
offers a convenient and natural way to incorporate the network
topology into our space transformation learning process. Second,
as being formulated in the vector quadratic form, it essentially
imposes the features/nodes selection through the coefficients of
$\mathbf{u}$ by shrinking those of irrelevant nodes toward zero
while crediting large values to those of relevant nodes. Indeed,
this quadratic $L_2$-norm is a kind of regularization which is
often referred to as the ridge shrinkage in statistics for
regression~\cite{HasTibFri01,DangANZS14}. The parameter $t$ is used to
control the amount of shrinkage. The smaller the value of $t$, the
larger the amount of shrinkage.

\vspace{-3ex}
\subsection{Solving the Function}           \label{sec:Solution}
\vspace{-1ex}

In order to solve our dual objective function associated with
constraints, we resort the Lagrange multipliers method and
following this, Eq.~\eqref{obj-Func4} can be rephrased as follows:
\vspace{-3ex}
\begin{align} \label{obj-F5}
\mathcal{L}(\mathbf{u},\lambda)= & \mathbf{u}^T \left(\mathbf{V
\widetilde{L}V}^T - \alpha \mathbf{C}\right) \mathbf{u} - \lambda
\left(\mathbf{u^T V D V^T u} - 1\right)
\end{align}
\vspace{-2ex}

\noindent of which, to simplify notations, we have denoted
$\mathbf{\widetilde{L}=L^{-}-L^{+}}$, $\mathbf{D}=\mathbf{D}^{+}$
and $\alpha$ is used in replacement for $t$ as there is a
one-to-one correspondence between them~\cite{HasTibFri01}. Taking
the derivative of $\mathcal{L}(\mathbf{u},\lambda)$ with respect
to vector $\mathbf{u}$ yields:

\vspace{-3ex}
\begin{align}
\frac{\partial
\mathcal{L}(\mathbf{u},\lambda)}{\partial\mathbf{u}}= 2
\left(\mathbf{V \widetilde{L}V}^T - \alpha \mathbf{C}\right)
\mathbf{u} - 2\lambda \mathbf{V D V^T u}
\end{align}
\vspace{-3ex}

\noindent And equating it to zero leads to the generalized
eigenvalue problem:

\vspace{-3ex}
\begin{align}
\left(\mathbf{V \widetilde{L}V}^T - \alpha \mathbf{C}\right)
\mathbf{u} = \lambda \mathbf{V D V^T u}
\end{align}
\vspace{-3ex}

It is noticed that $\mathbf{V}$ is a singular matrix and its rank
is at most $\min(n,m)$, making the combined matrix on the right
hand side not directly invertible. We therefore decompose
$\mathbf{VD}^{1/2}$ into $\mathbf{P\Sigma Q}^T$, where columns in
$\mathbf{P}$ and $\mathbf{Q}$ are respectively called the left and
right (orthonormal) singular vector of $\mathbf{VD}^{1/2}$ while
$\mathbf{\Sigma}$ stores its singular values. Note that this
decomposition is always possible since $\mathbf{D}$ is a
non-negative diagonal matrix of node degrees. Additionally, both
$\mathbf{P}$ and $\mathbf{Q}$ can be represented in the square
matrices while $\mathbf{\Sigma}$ a rectangular one of $n\times m$
size according to the most general decomposition form
in~\cite{Dhillon2006}. Following this, the combined matrix on the
right hand size can be rewritten as:

\vspace{-4ex}
\begin{align}
\mathbf{V D V}^T = \mathbf{P\Sigma}^2 \mathbf{P}^T
\end{align}
\vspace{-4ex}

\noindent And in order to get a stable solution, we keep the top
ranked singular values in $\mathbf{\Sigma}$ such as their
summation explains for no less than 95\% of the total singular
values\footnote{Note that since
$(\mathbf{VD}^{1/2})(\mathbf{VD}^{1/2})^T$ is Hermitian and
positive semidefinite, the diagonal entries in $\mathbf{\Sigma}$
are always real and nonnegative.}. Let us denote
$\mathbf{B^{\ast}}=\mathbf{P\Sigma}^{-2} \mathbf{P}^T$ as the
inversion of the right hand side and before showing our optimal
solution, we need the following proposition:

\begin{proposition} \label{Prop1}
Let $\mathbf{P}$ be the matrix of left singular vectors of
$\mathbf{VD^{1/2}}$ defined above, then its row vectors are also
orthogonal, i.e., $\mathbf{P}\mathbf{P}^T=\mathbf{I}$
\end{proposition}

\begin{proof}
Let $\mathbf{a}$ be an arbitrary vector, we need to show
$\mathbf{P}\mathbf{P}^T\mathbf{a}=\mathbf{a}$. Due to the
orthogonal property of left singular vectors, it is true that
$\mathbf{P}^T\mathbf{P}=\mathbf{I}$. The inversion of $\mathbf{P}$
therefore is equal to $\mathbf{P}^T$ and given arbitrary vector
$\mathbf{a}$, there is a uniquely determined vector $\mathbf{b}$
such that $\mathbf{P}\mathbf{b}=\mathbf{a}$. Consequently,

\vspace{-2ex}
$$\mathbf{P}\mathbf{P}^T\mathbf{a}=\mathbf{PP^TP}\mathbf{b}=\mathbf{P}\mathbf{b}=\mathbf{a}$$
\vspace{-2ex}

It follows that $\mathbf{P}\mathbf{P}^T=\mathbf{I}$ since
$\mathbf{a}$ is an arbitrary vector.
\end{proof}

\begin{theorem} \label{Theo1}
Given $\mathbf{B}= \mathbf{P\Sigma}^{2} \mathbf{P}^T$, we have
$\mathbf{BB}^{\ast}=\mathbf{I}$
\end{theorem}

\begin{proof}
The proof of this theorem is straightforward given
Proposition~\ref{Prop1}.
\end{proof}
\vspace{-1ex}

Now, for simplicity, let us denote $\mathbf{A}$ for the combined
matrix $(\mathbf{V \widetilde{L}V}^T - \alpha \mathbf{C})$, then
it is straightforward to see that $\mathbf{u}$ turns out to be the
eigenvector of the equation: \vspace{-3ex}
\begin{align} \label{Eq-6}
\mathbf{B^{\ast}A} = \lambda \mathbf{u}
\end{align}
\vspace{-3ex}

\noindent with the maximum value is given by the following
theorem.

\begin{theorem} \label{Theo2}
Given matrix $\mathbf{A}=\mathbf{V \widetilde{L}V}^T - \alpha
\mathbf{C}$ and $\mathbf{B}=\mathbf{V D V}^T$ defined above, the
maximum value of $\mathbf{u}^T\mathbf{Au}$ subjected to
$\mathbf{u}^T\mathbf{Bu}=1$ is the largest eigenvalue of
$\mathbf{B^{\ast}A}$.
\end{theorem}

\begin{proof}
Due to Theorem~\ref{Theo1}, it is straightforward to see that:
$$\mathbf{u}^T\mathbf{Au}= \mathbf{u}^T\mathbf{BB^{\ast}Au}$$
On the other hand, $\mathbf{u}^T\mathbf{BB^{\ast}Au}=
\mathbf{u}^T\mathbf{B}\lambda \mathbf{u}$ by equation
Eq.~\eqref{Eq-6} and further taking into account our second
constraint, it follows that:
$$\label{Eq-7}
\max_{\mathbf{u}:\mathbf{u}^T\mathbf{Bu}=1}
\{\mathbf{u}^T\mathbf{Au}\} = \max \{\lambda\}
$$
%
\end{proof}

From this theorem, it is safe to say that
$\mathbf{u}^*=\mathbf{u}_1$ as the first eigenvector of
$\mathbf{B^{\ast}A}$ corresponding to its largest eigenvalue
$\lambda_1$ is our optimal solution. Since eigenvectors and
eigenvalues go in pair, the second optimal solution is the second
eigenvector $\mathbf{u}_2$ corresponding to the second largest
eigenvalue $\lambda_2$ and so on. Consequently, in the general
case, if $d$ is the number of unique global network states, our
optimal transformed space is the one spanned by the top $d$
eigenvectors. In the next section, we present a method to select
optimal features/nodes along with the subnetworks formed by these
nodes.

\vspace{-2ex}
\subsection{Subnetwork Selection}           \label{sec:subnetwork}
\vspace{-0ex}

In essence, our top $d$ eigenvectors play the role of space
transformation which projects network data from the original high
dimensional space into the induced subspace of $d$ dimensions.
Their coefficients essentially reflect how the original nodes
(features) have been combined or more specifically, the degree of
node's importance in contributing to the subspace that optimally
discriminates network instances. Following the approach adopted
in~\cite{DangMAN13} with $c$ as the user parameter, we
select top $c$ entries in each $\{\mathbf{u}_i\}_{i=1}^d$
corresponding to the selective nodes. Nonetheless, it is possible
that there will be more than $c$ nodes selected by combining from
$d$ eigenvectors. Therefore, in practice, we may use a simple
approach by first selecting the largest absolute entries across
$d$ eigenvectors: 

\vspace{-1ex}
\begin{align} \label{Eq-8}
\mathbf{v}=\{v_1,\ldots, v_n\} \text{ where } v_p=\
\operatorname*{max}_{i} |u_{i,p}|
\end{align}
\vspace{-3ex}

\noindent where $u_{i,p}$ is the $p$-th entry of eigenvector
$\mathbf{u}_i$, and then selecting nodes according to the top $c$
ranking entries in $\mathbf{v}$. The subnetworks forming from
these nodes can be straightforwardly obtained by matching to the
nodes in our generalized network $N$ defined in
Definition~\ref{Def2}, along with their connecting edges stored in
$E$. These subnetworks can be visualized which offers the user an
intuitive way to examine the results.

\vspace{-3ex}
\subsection{Computational Complexity}           \label{sec:Complexity}
\vspace{-2ex}

We name our algorithm \verb"SNL", an acronym stands for SubNetwork
spectral Learning. Its computation complexity is analyzed as
follows. We first need to compute edges' weights according to
Definition~\ref{Def2} to build our first meta-graph which takes
$O(n^2m)$ since there are at most $n(n-1)/2$ edges in the
generalized network $N$. Second, in building the two subsequent
meta-graphs, the cosine distance between any two network instances
is computed which amounts to $O(n^2m)$ or $O(mn\log n)$ in case
the multidimensional binary search tree is used~\cite{Ben75}.
Also, since the size of matrix $\mathbf{VD}^{1/2}$ is $m\times n$,
its singular value decomposition takes $O(mn\log n)$ with the
Lanczos technique~\cite{Golub1996}. Likewise, the
eigen-decomposition of the matrix $\mathbf{B^{\ast}A}$ takes
$O(n^2\log n)$ since its size is $n\times n$. Therefore, in
combination, the overall complexity is at most $O(n^2m+ n^2\log
n)$ assuming that the number of nodes is larger than the number of
network instances.

\vspace{-3ex}
\section{Empirical Studies}           \label{sec:Experiment}
\vspace{-2ex}

\subsection{Datasets and Experimental Setup}           \label{sec:Exp.setup}
\vspace{-2ex}

We compare the performance of \verb"SNL" against
\verb"MINDS"~\cite{RanuHS13} which is among the first approaches
formally addressing the global-state network classification
problem by a subnetwork sampling. Another algorithm for comparison
is the Network Guided Forests (\verb"NGF")~\cite{DutkowskiI11}
designed specifically for protein protein interaction (PPI)
networks. We use both synthetic and real world datasets for
experimentation. Since global states are available in all
datasets, we compare average accuracy in $10$-fold cross
validation for synthetic data, and $5$-fold cross validation for
real data (due to smaller numbers of network instances). For
\verb"SNL", the cross validation is further used to select its
optimal $\alpha$ parameter (shortly discussed below). Unless
otherwise indicated, we set $k=10$ and use the linear-SVM to
perform training and testing in the transformed space (keeping top
50 nodes) in \verb"SNL". We set \verb"MINDS"' parameters as
follows: $10000$ sampling iterations, $0.8$ discriminative
potential threshold and $K=200$ as recommended in the original
paper~\cite{RanuHS13}. The Gini index is used for the tree
building in \verb"NGF" and we set its improvement threshold
$\epsilon=0.02$~\cite{DutkowskiI11}.

\vspace{-2ex}
\subsection{Results on Synthetic Datasets}           \label{sec:SyntheticData}

We use synthetic data to evaluate the performance of our technique
in training robust classifiers and selecting relevant subnetworks.
We generate scale-free backbone networks by preferential
attachment of a predefined size adding $20$ edges for each new
node. The probabilities of backbone edges are sampled from a
truncated Gaussian distributions: $N(0.9,0.1)$ for edges among
\textit{ground truth nodes} (pre-selected nodes of
high-correlation with the network state) and $N(0.7,0.1)$ for the
rest of the edges. The weighted backbone serves as our generalized
template to generate network instances by independently sampling
the existence of every edge based on its probability. The global
states are binary $S_i\in\{0,1\}$ with balanced distribution. We
further add noise to both global and local states of ground truth
nodes, respectively with levels of $10\%$ and $30\%$.

\paragraph{Varying $|V_{gt}|$:}In the first set of experiments, we aim to
test whether the performance of all algorithms is affected by the
number of ground truth nodes. To this end, we generate 5 datasets
by fixing $m=1000$ instances, $n=3000$ nodes and vary the ground
truth nodes $|V_{gt}|$ from 10 to 50. In
Figure~\ref{fig:synGTFeaACC}, we report the average accuracy (and
standard deviation) of all algorithms in 10-fold cross validation.
As one may observe, \verb"SNL" performs stably regardless of the
change in the ground truth sizes. Compared to the other
techniques, its classification is always consistently higher
across all cases. The \verb"MINDS" technique also performs well on
this experimental setting yet the \verb"NGF" seems to be sensitive
to the small ground truth sizes. For small $|V_{gt}|$, the
sampling strategy based on density areas employed in \verb"NGF"
has little chance to select the ground truth nodes, making its
accuracy close to a random technique. When more ground truth nodes
are introduced, \verb"NGF" has higher possibility to sample
high-utility nodes and in the last two datasets, its performance
is on par with that of \verb"MINDS". Nonetheless, its accuracy
only peaks at 73\% in the best case which is lower than 77\% in
\verb"SNL" (last column).

\vspace{-2ex}
\begin{figure}[t]
  \begin{minipage}[b]{0.5\linewidth}
    \centering
    \includegraphics[width=\linewidth]{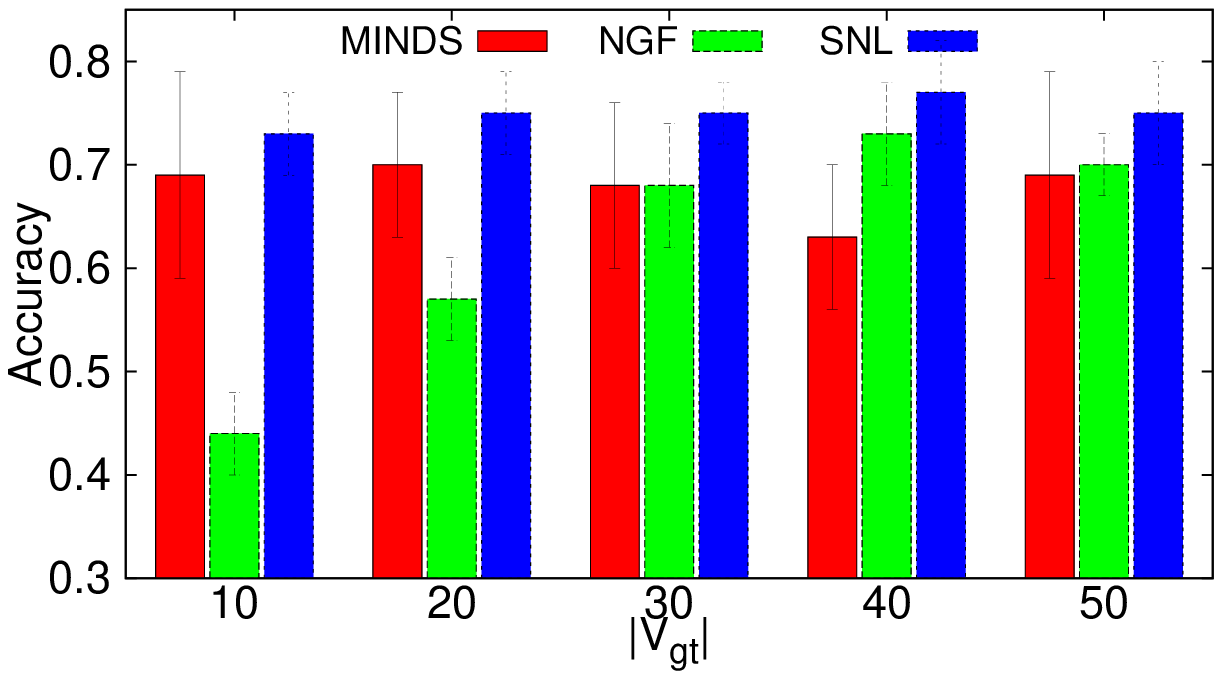}
    \vspace{-6ex}
    \caption{Accuracy of all algorithms by varying ground truth subnetworks' nodes}
    \label{fig:synGTFeaACC}
  \end{minipage}
  \hspace{0.5cm}
  \begin{minipage}[b]{0.5\linewidth}
    \centering
    \includegraphics[width=\linewidth]{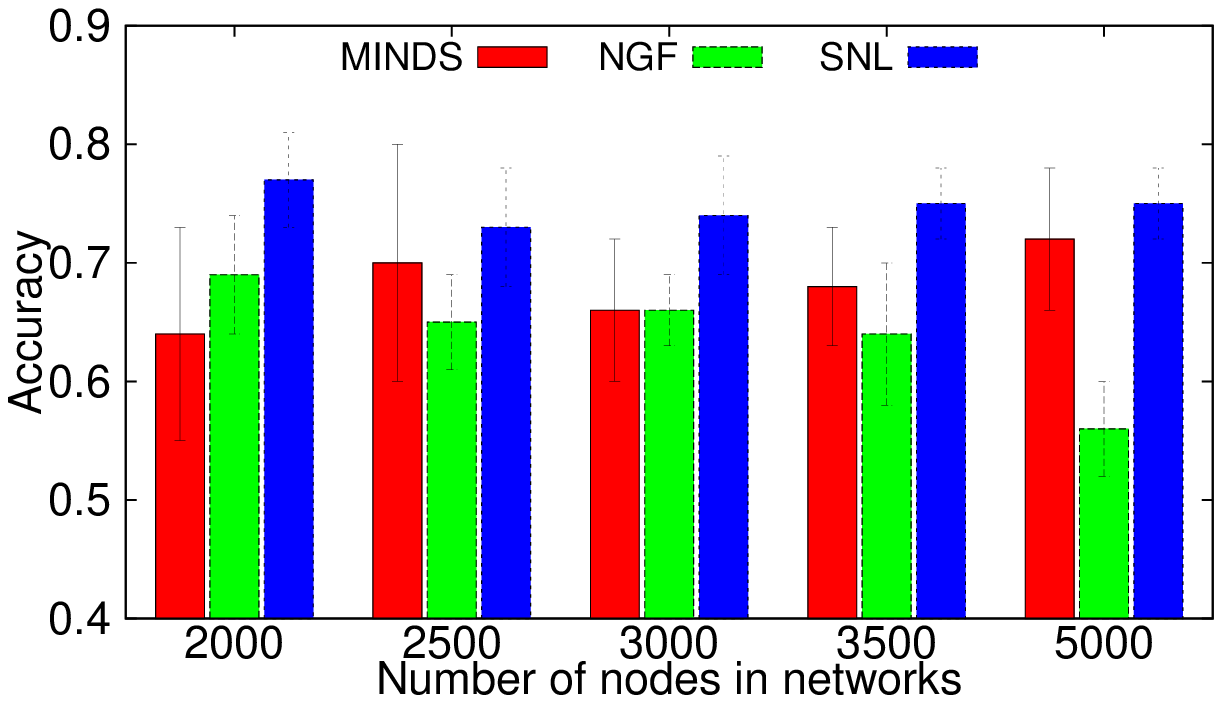}
    \vspace{-6ex}
    \caption{Accuracy of all algorithms by varying network size}
    \label{fig:synFeaACC}
  \end{minipage}
  \vspace{-6ex}
\end{figure}

\vspace{-1ex}
\paragraph{Varying network size:}
In the second set of experiments, we evaluate the performance of
all algorithms by varying the network sizes. Specifically, we fix
$m=3000$ network instances, $|V_{gt}|= 50$ ground truth nodes and
generate 5 datasets having the network size varied from 2000 to
5000 nodes. The classification performance along with the standard
deviation is reported in Figure~\ref{fig:synFeaACC}. It is
possible to see that the performance traits are similar to those
in our first set of experiments. \verb"SNL"'s classification
accuracy remains high while that of \verb"NGF" decreases with the
increase of network size. This again can be explained by the
extension of the searching subnetwork space, leading to the lower
likelihood of both \verb"NGF" and \verb"MINDS" in identifying
relevant subnetworks with potentially discriminative nodes. The
slightly better performance of \verb"MINDS" (compared to
\verb"NGF") is due to its accuracy thresholding in selecting
candidate substructures. The set of \verb"MINDS"' selected trees
are thus qualitatively better. Nonetheless, as compared to
\verb"SNL", our subspace learning approach show more competitive
results. Moreover, since the low-dimensional subspace learnt in
\verb"SNL" is unique and linearly combined from the most
discriminative nodes, its performance also shows more stable,
indicated by the small standard deviation across all cases.

\begin{figure}[t]
\centering
    \subfigure[]{
        \includegraphics[width=0.45\linewidth]{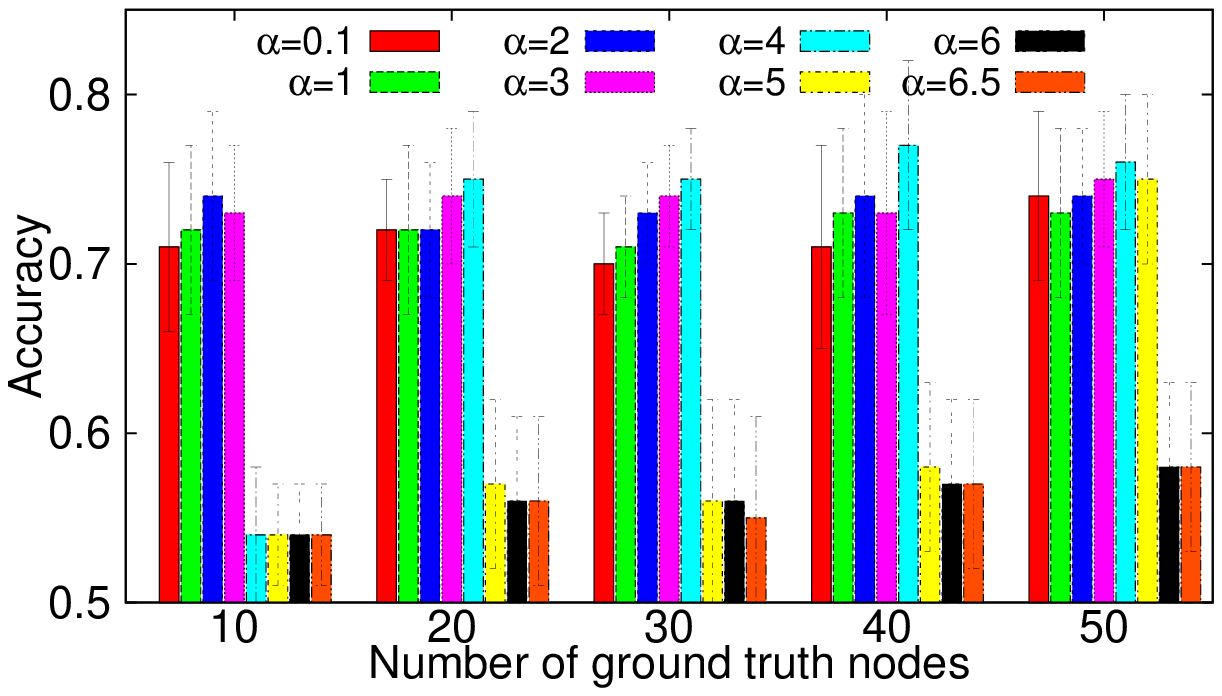}
    \label{fig:synGTFeaACCalpha}
    }
    \subfigure[]{
    \includegraphics[width=0.45\textwidth]{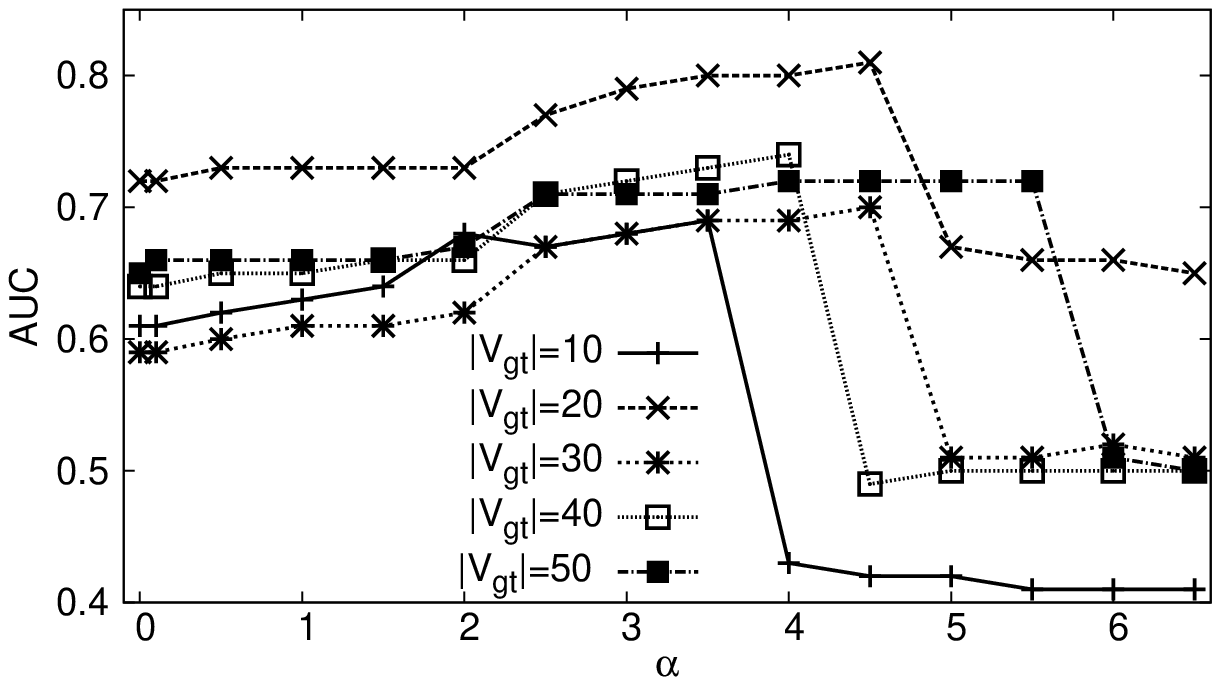}
    \label{fig:synGTFeaAucAlpha}
} \vspace{-4ex} \caption{Network effect on accuracy
\subref{fig:synGTFeaACCalpha} and AUC performance
\subref{fig:synGTFeaAucAlpha} for different numbers of ground
truth nodes.}\label{fig:synGTFeaAlpha} \vspace{-3ex}
\end{figure}

\begin{figure}[t]
\centering
    \subfigure[]{
        \includegraphics[width=0.47\linewidth]{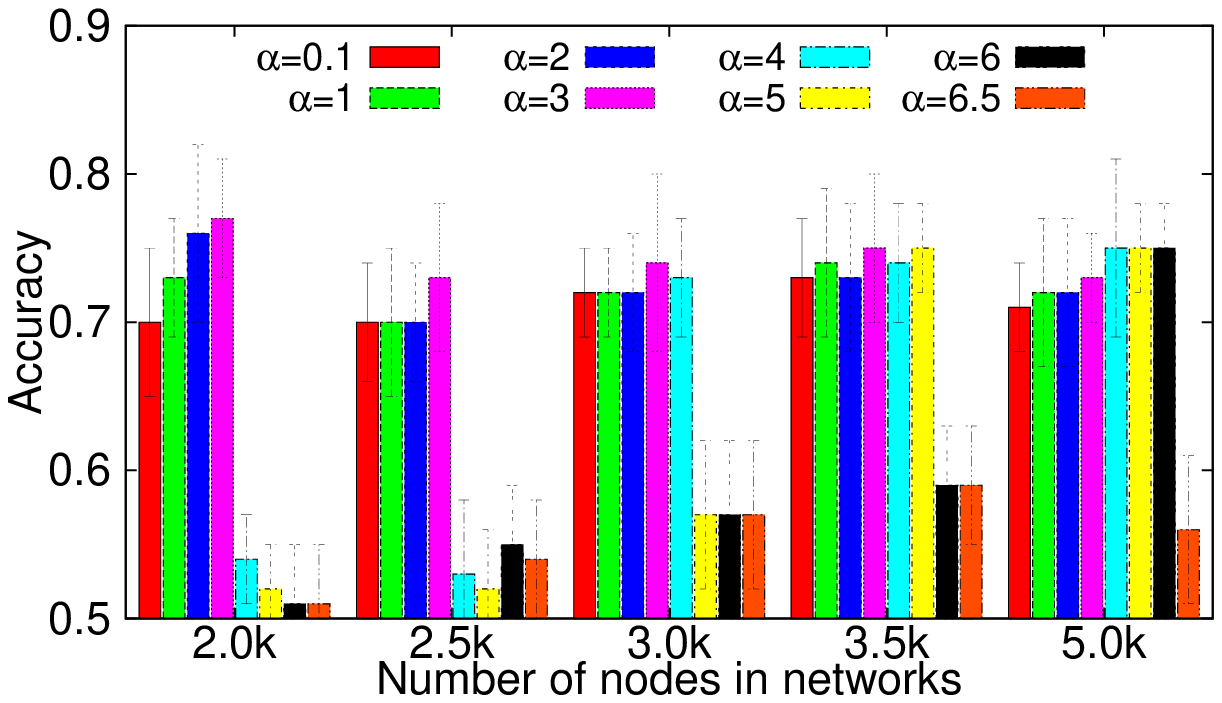}
    \label{fig:synFeaACCalpha}
    }
    \subfigure[]{
    \includegraphics[width=0.47\textwidth]{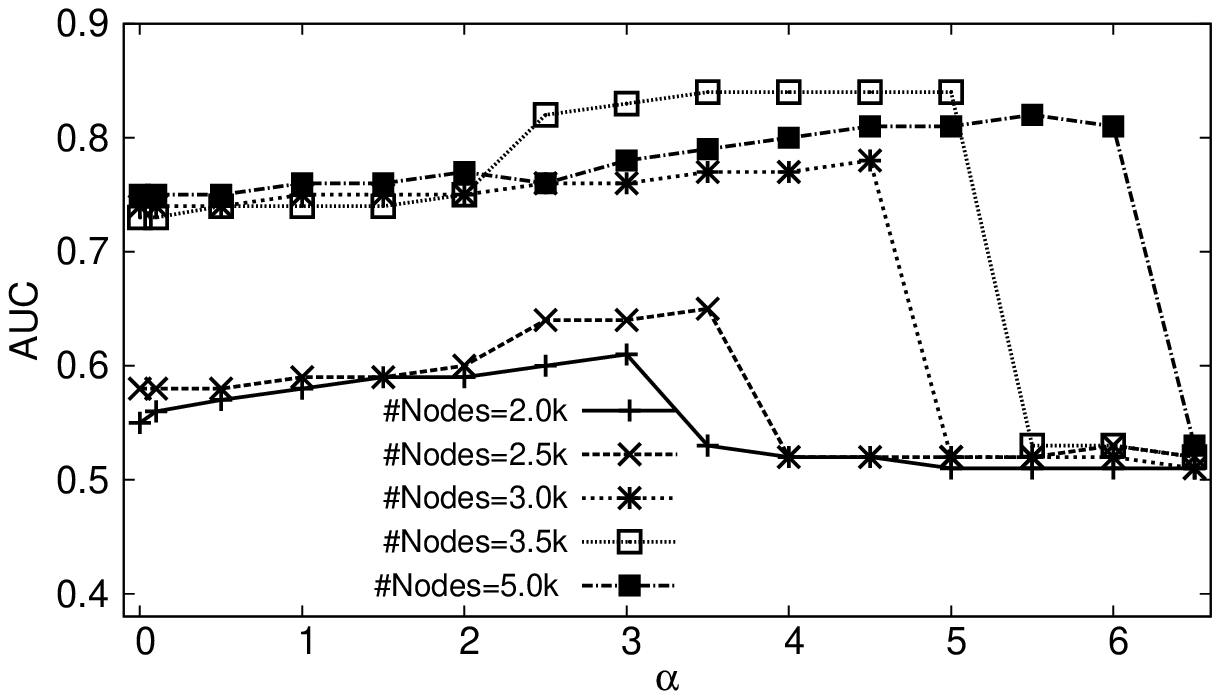}
    \label{fig:synFeaAucAlpha}
} \vspace{-4ex} \caption{Network effect on accuracy
\subref{fig:synGTFeaACCalpha} and AUC
\subref{fig:synGTFeaAucAlpha} for different network
sizes.}\label{fig:synFeaalpha} \vspace{-4ex}
\end{figure}

\vspace{-2ex}
\paragraph{Effect of network topology:} In order to provide more insights into the
performance of \verb"SNL", we further test the network effect. As
presented in Section~\ref{sec:Approach}, $\alpha$ is the parameter
controlling the influence of the network information on the
subspace learning process. The higher the $\alpha$, the more
preference putting on the heavily connected nodes. We report in
Figures~\ref{fig:synGTFeaACCalpha},\ref{fig:synFeaACCalpha} the
accuracy of \verb"SNL" by varying $\alpha$ from $0.1$ to $6.5$ and
in Figures~\ref{fig:synGTFeaAucAlpha},\ref{fig:synFeaAucAlpha} its
ability in discovering the ground truth nodes. For the latter
case, we validate the performance through the usage of area under
the ROC curve (AUC)~\cite{HasTibFri01}.

As expected, incorporating the network structure in the subspace
learning process improves both classification rate and the AUC in
uncovering the ground truth nodes. The plots in
Figures~\ref{fig:synGTFeaACCalpha},\ref{fig:synFeaACCalpha} show
that the accuracy initially improves for increasing influence of
the network ($\alpha\leq5$) and then decreases as the network
component becomes prevalently dominant ($\alpha>5$). This is
because for large $\alpha$, \verb"SNL" tends to incorporate
irrelevant nodes solely based on their strong connections to the
neighbors (yet their local values might not help classifying
global state values). Another notable observation is that, in
larger instances or ground truth feature sets, the optimal
$\alpha$ tends to increase as well. Moreover, the values of
$\alpha$ that maximize classification accuracy also result in
optimal AUC in identifying the ground truth nodes
(Fig.~\ref{fig:synGTFeaAucAlpha},\ref{fig:synFeaAucAlpha}). These
experiments clearly show the helpful information provided by the
network topology in uncovering the groundtruth features. Also, we exclude \verb"NGF" and \verb"MINDS" from these experiments (to save space) and leave the discussion over their AUC performance with the real-world datasets.

%

\vspace{-2ex}
\subsection{Real-world Datasets}           \label{sec:PPI data}
\vspace{-1ex} 

We use $4$ real-world datasets to evaluate the performance of
\verb"SNL" and its competing methods. The features in all datasets
correspond to micro-array expression measurements of genes; the
topology structures relating features correspond to gene
interaction networks; and the global network states correspond to
phenotypic traits of the subjects/instances. The statistics of our
datasets are listed in Table~\ref{tbl:network_stats}. Two of our
real-world datasets, breast cancer and embryonic development, were
also used for experimentation in the original \verb"NGF" method
~\cite{DutkowskiI11}. Our other datasets come from a study on
maize properties~\cite{Hui13} and a human liver metastasis
study~\cite{Dong07} combined with a functional
network~\cite{dannenfelser2012genes2fans}. The network samples are
used as provided in the original studies, except for maize where
we down-sample one of the classes to balance the global state
distribution.

\begin{table}[t]
 \caption{Real-world dataset statistics and sources} \vspace{-2ex}
  \label{tbl:network_stats}
   \renewcommand{\arraystretch}{1.3}
    \centering
    {\scriptsize
    \begin{tabular}{ |l|c|c|c|l|}
    \hline
    \textbf{Datasets} & \textbf{Genes} & \textbf{Edges} & \textbf{Instances} & \textbf{Global State} \\
    \hline
    \textbf{Breast cancer} & $11203$ & $57235$ & $295$ & cancer/non-cancer \\ \hline
    \textbf{Embryonic development} & $1321$ & $5227$ & $35$ & developmental tissue layer \\ \hline
    \textbf{Maize} & $8574$ & $298510$ & $344$ & high/low oil production\\ \hline
    \textbf{Liver metastasis} & $7383$ & $251916$ & $123$ & disease/non-disease \\ \hline

    \end{tabular} \vspace{-2ex}
    }
\end{table}

\vspace{-2ex}
\paragraph{Classification performance:}
The comparison of classification accuracy for all techniques and
datasets is presented in Figure~\ref{fig:realACC}. We report the
average accuracy and standard deviation from the 5-fold stratified
cross validation. All techniques perform competitively on the
breast cancer data, achieving more than 70\% of classification
accuracy on average. The accuracy of \verb"SNL" dominates
significantly that of the sampling techniques on the embryonic and
maize datasets (at least $15\%$ and $10\%$ improvement
respectively) and less so in the liver dataset. The separation is
highest in the datasets of small number of instances and big
number of feature nodes -- the settings in which \verb"SNL" is
particularly effective. Beyond average performance improvement,
\verb"SNL"'s accuracy is also more stable across all folds as it
considers the global network structure when learning a subspace
for classification, while the alternatives perform sampling in the
exponential space of substructures.

\begin{figure}[t]
\centering
    \subfigure[]{
    \includegraphics[width=0.34\linewidth]{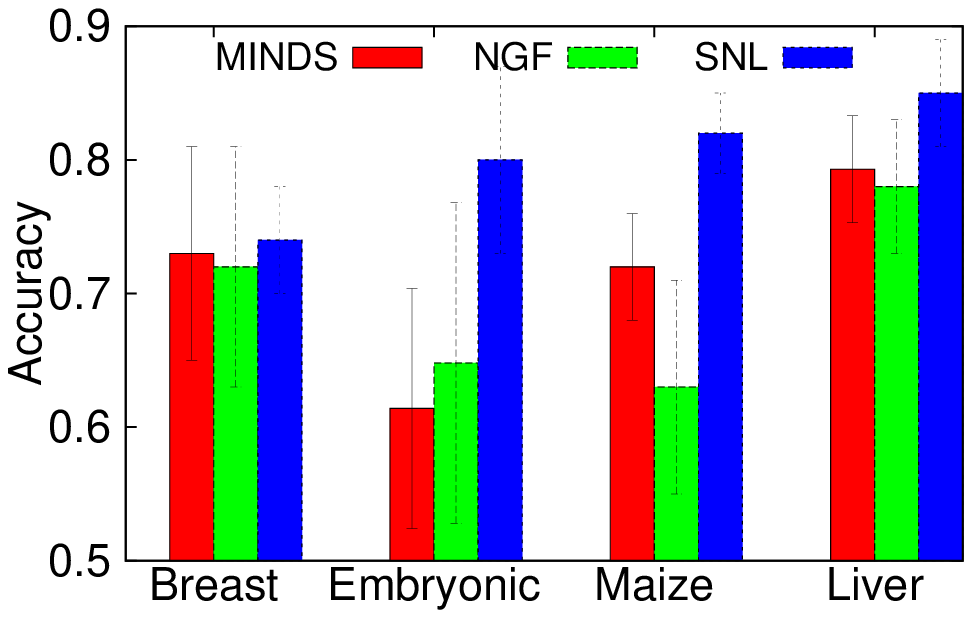}
    \label{fig:realACC}
    } \hspace{-0.5cm}
\subfigure[]{
    \includegraphics[width=0.30\textwidth]{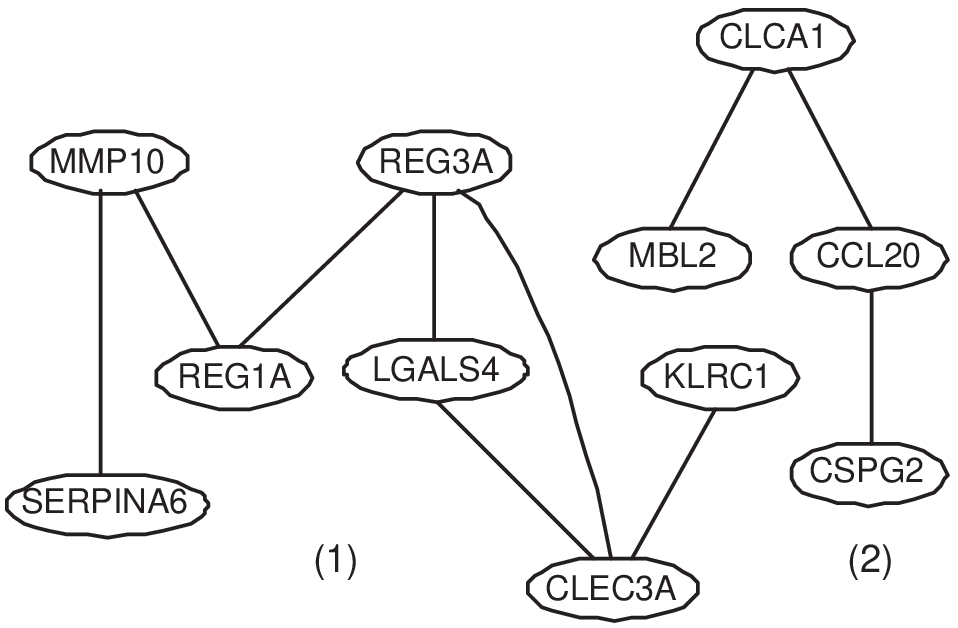}
    \label{fig:LiverSubnet}
}\hspace{-0.4cm}
\subfigure[]{
    \includegraphics[width=0.35\textwidth]{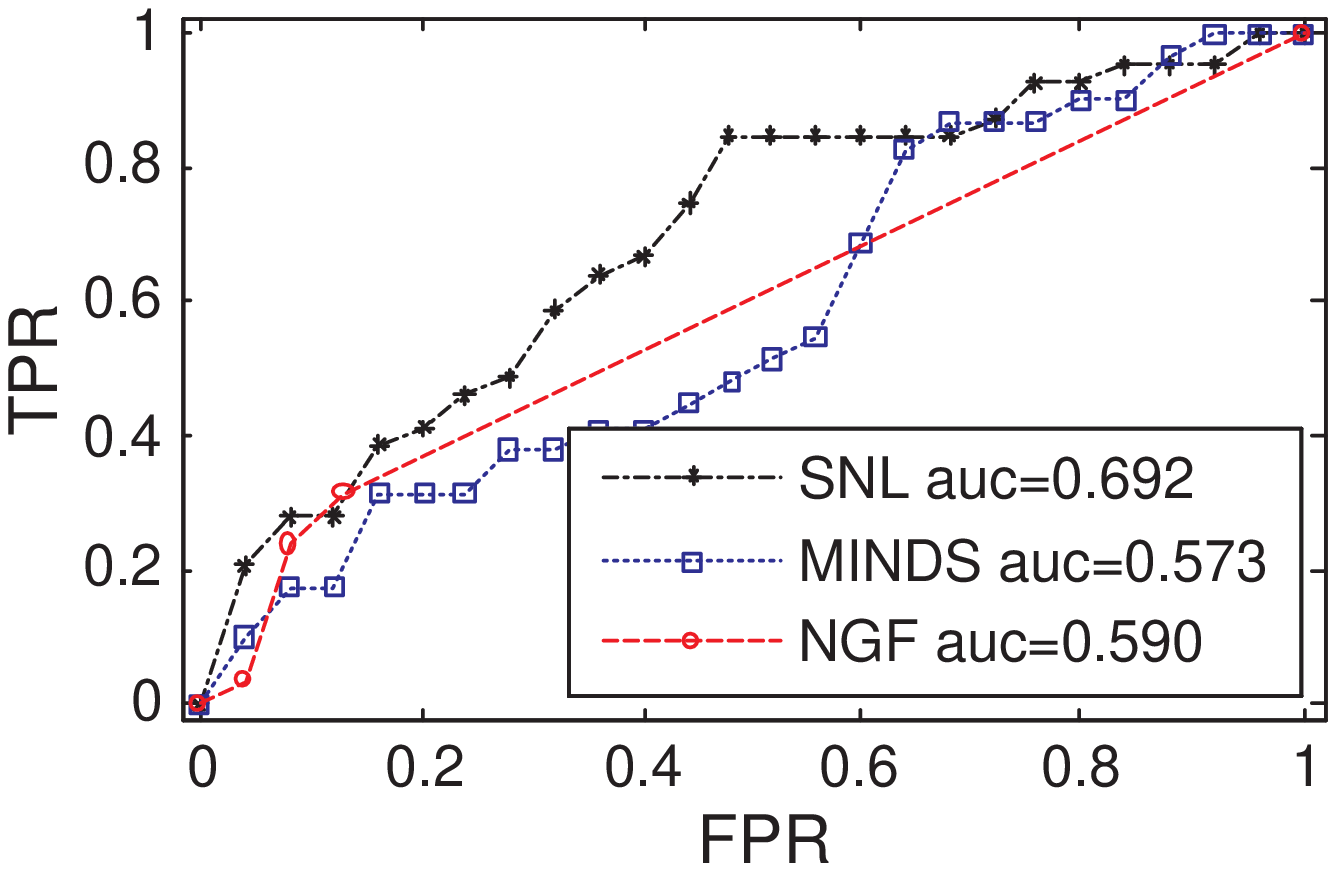}
    \label{fig:liverROC}
} \hspace{-0.5cm} \vspace{-4ex} \caption{\subref{fig:realACC}
Classification performance of all algorithms on four real world
datasets. \subref{fig:LiverSubnet} Subnetworks identified by the
SNL related to the liver metastasis. \subref{fig:liverROC} ROC
performance over the liver metastasis ($x$-axis is false positive
rate, $y$-axis is true positive rate). } \vspace{-4ex}
\end{figure}

\vspace{-3ex}
\paragraph{Subnetwork discovery:}
Unlike the synthetic datasets where we can control the ground
truth network features, it is generally much harder to obtain
ground truth subnetworks for real world datasets. However, as an
attempt to look deeper into the results, we choose the Liver
metastasis and further investigate the meaningful subnetworks
generated by the \verb"SNL". For this dataset, out of top 50 nodes
of highest coefficient values (ref. Section~\ref{sec:subnetwork}),
about one third of the nodes are connected into four subnetworks.
We depict in Figure~\ref{fig:LiverSubnet} the two largest ones
which respectively contain 7 and 4 connected gene nodes. Among
these selected subnetworks, the genes \verb"REG1A" and
\verb"REG3A" are particularly interesting since they are in
agreement with the ones found in~\cite{Dong07} which was shown to
be involved in the liver metastasis cancer. As a comparison
against \verb"MINDS" and \verb"NGF", we notice that both methods
generate multiple binary-trees where each node has only a
\textit{single} parent. Moreover, while \verb"SNL" can provide a
natural rank of important genes based on their coefficients (from
the learnt subspace), it is less trivial to define important genes
from \verb"NGF" and \verb"MINDS" as they both generate thousands
of trees. For the purpose of measuring biological relevance of
obtained genes, we define a ranking for these competing techniques
based on the frequency of genes appeared in the generated trees.
For comparison, we select $46$ metastasis-specific genes
identified in~\cite{Dong07} to serve as a ground truth set (39
intersect with our network and expression data) and plot the ROC
performance of all algorithms in Figure~\ref{fig:liverROC}. Note
that, this is only a partial ground truth set, since identifying
all genes associated with this disease is a subject of ongoing
research~\cite{Dong07}. It is observed that the ranking produced
by \verb"SNL" includes more ground truth genes than those of
\verb"NGF" and \verb"MINDS" at increasing false-positive rates.
The higher true positive rates of \verb"SNL" makes it a better
method for identifying new genes associated with the phenotype of
interest. In practice, this is an important feature of the
algorithm since validating even a single gene related to cancer is
both time-wise and financially costly. As shown in
Figure~\ref{fig:liverROC}, while the ROC performance of \verb"NGF"
and \verb"MINDS" are only at $0.59$ and $0.57$ AUC, that value of
\verb"SNL" is $0.69$ which clearly demonstrates large gap of
better performance.

\vspace{-3ex}
\section{Related Work}           \label{sec:related work}
\vspace{-2ex}

Mining discriminative subspaces from global-state networks is a
novel and challenging problem. Two lines of work close to this
problem are network classification and mining evolving subgraphs
from dynamic network data. In the network classification case,
most representative algorithms are \verb"LEAP"~\cite{YanCHY08},
\verb"graphSig"~\cite{RanuS09}, \verb"GAIA"~\cite{JinYW10} and
\verb"COM"~\cite{JinYW09} which generally assume a database
consisting of positive and negative networks that need to be
classified. These approaches, though diverse in terms of their
underlying algorithms, all aim at extracting a set significant
subnetworks that are \textit{more frequent} in one class of
positive networks and \textit{less frequent} in the negative
class. Different from the above problems, we aim to mine
subnetworks which are represented in all network instances; yet
the node values along with the network structures can discriminate
the global states of the networks. Another line of related
research focuses on mining dynamic evolving
subnetworks~\cite{MongioviBS13,BogdanovFMPRS13,BogdanovMS11}. The
problem in this case is to obtain subnetworks over time that
evolve significantly (outliers) from other network locations. This
setting therefore do not model the problem developed in this paper
since the dynamic network snapshots neither contain global-state
values nor can remove their temporal property.

Several studies in systems biology have indicated the critical
role of the network structure in identifying protein modules
related to clinical outcomes, for both
regression~\cite{LiL10,NoirelSW08,LiL08a} and
classification~\cite{DutkowskiI11,RanuHS13}. In the classification
setting which is related to our study, the
\verb"NGF"~\cite{DutkowskiI11} is an ensemble approach that builds
a forest of trees jointly voting for the class of a network
instance. Resided at the \verb"NGF"'s core is the \verb"CART"
(classification and Regression tree) technique and in order to
build a decision tree within the PPI network, \verb"NGF" starts
with a root node and progressively includes connected nodes as
long as the improvement in class separation (measured by Gini
index) is no smaller than a given threshold. The study
in~\cite{RanuHS13} is the first one to formally introduce the
problem of subnetwork mining in global-state networks and further
propose the \verb"MINDS" algorithm to solve it. Similar to
\verb"NGF", \verb"MINDS" adopts network-constraint decision trees
and is also an ensemble classifier. Nonetheless, it increases the
quality of decision trees by developing a novel concept of editing
map over the space of potential subnetworks and exploits Monte
Carlo Markov Chain sampling over this novel data structure to seek
decision trees with maximum classification potential. Unlike the
frequency-based and sampling classification discussed above, our
approach is fundamentally different as it searches for the most
discriminative subnetworks in a single low dimensional subspace
through the spectral learning technique, which generally leads to
more stable and high-accuracy performance.

\vspace{-2ex}
\section{Conclusion}           \label{sec:conclusion}
\vspace{-1ex}

We proposed a novel algorithm named \verb"SNL" to address the
challenging problem of uncovering the relationship between local
state values residing on nodes and the global network events.
While most existing studies address this problem by sampling the
exponential subnetworks space, we adopt an efficient and effective
subspace transformation approach. Specifically, we define three
meta-graphs to capture the essential neighboring relationships
among network instances and devise a spectral graph theory
algorithm to learn an optimal subspace in which networks with
different global-states are well separated while the common
structure across samples is smoothly respected to enable
subnetwork discovery. Through experimental analysis on synthetic
data and real-world datasets, we demonstrated its appealing
performance in both classification accuracy and the real-world
relevance of the discovered discriminative subnetwork features.


\noindent \textbf{Acknowledgements}: The research work was supported in part by the NSF (IIS-1219254) and the NIH (R21-GM094649).
\bibliographystyle{abbrv}
\bibliography{paper}
\end{document}